\documentclass[sigconf, nonacm]{acmart}
\usepackage{xcolor}
\usepackage{xspace}
\usepackage{enumitem}
\usepackage{amsmath}
\usepackage{cancel}
\usepackage{algorithm}
\usepackage{algpseudocode}
\usepackage{graphicx}
\usepackage{subcaption}
\usepackage{tabularx}

\usepackage{tcolorbox}
\tcbuselibrary{skins,breakable}

\definecolor{PromptAccent}{HTML}{284B70}
\definecolor{PromptFrame}{HTML}{C6D3DF}
\definecolor{PromptBackground}{HTML}{F7F9FC}

\newtcolorbox{preferprompt}[2]{
    enhanced,
    breakable,
    width=\linewidth,
    colback=PromptBackground,
    colframe=PromptFrame,
    colbacktitle=PromptAccent,
    coltitle=white,
    title={Stage (#1): #2},
    title after break={Stage (#1): #2 (continued)},
    fonttitle=\sffamily\bfseries\small,
    fontupper=\small,
    fontlower=\small,
    boxrule=0.4pt,
    arc=1mm,
    boxsep=0pt,
    left=7pt,
    right=7pt,
    top=5pt,
    bottom=5pt,
    toptitle=4pt,
    bottomtitle=4pt,
    before skip=7pt,
    after skip=7pt,
    segmentation style={solid,draw=PromptFrame,line width=0.4pt},
    before upper={\setlength{\parindent}{0pt}%
                  \setlength{\parskip}{3pt}\raggedright},
    before lower={\setlength{\parindent}{0pt}%
                  \setlength{\parskip}{3pt}\raggedright}
}

\newlist{promptsteps}{enumerate}{1}
\setlist[promptsteps]{
    label=\arabic*.,
    leftmargin=1.5em,
    labelsep=0.45em,
    topsep=0pt,
    itemsep=0pt,
    parsep=0pt,
    partopsep=0pt
}

\newtcolorbox{pycodebox}{
    colback=gray!5,
    colframe=gray!50,
    boxrule=0.5pt,
    arc=2pt,
    left=5pt,
    right=5pt,
    top=5pt,
    bottom=5pt
}

\AtBeginDocument{%
  }

\newtheorem{theorem}{Theorem}

\newtheorem{proposition}[theorem]{Proposition}
\newtheorem{lemma}[theorem]{Lemma}
\newtheorem{corollary}[theorem]{Corollary}

\newcommand{\prefer}{\textsc{Prefer}\xspace}
\newcommand{\amazon}{\textit{Amazon}\xspace}
\newcommand{\reviews}{\textit{Reviews'23}\xspace}

\begin{document}

\title{\prefer: Personalized Review Summarization with Online Preference Learning}

\author{Millend Roy}
\affiliation{%
  \institution{Columbia University}
  \city{New York}
  \state{NY}
  \postcode{10027}
  \country{USA}
}

\author{Agostino Capponi}
\affiliation{%
  \institution{Columbia University}
  \city{New York}
  \state{NY}
  \postcode{10027}
  \country{USA}
}

\author{Vineet Goyal}
\affiliation{%
  \institution{Columbia University}
  \city{New York}
  \state{NY}
  \postcode{10027}
  \country{USA}
}

\begin{abstract}
\label{sec:abstract}
Product reviews significantly influence purchasing decisions on e-commerce platforms. However, the sheer volume of reviews can overwhelm users, obscuring information most relevant to their individual needs. Existing e-commerce summarization systems typically produce generic, static summaries that fail to account for the fact that (i) different users prioritize different product characteristics, and (ii) these preferences may evolve with interactions. To address these challenges, we propose an online learning framework that generates personalized summaries for each user by representing user interests in a latent aspect space and selecting preference-aligned review evidence. Our system then iteratively refines its understanding of user preferences by incorporating feedback on the generated summaries directly over time.
We evaluate our framework on \textit{Amazon Reviews'23} using controlled feedback simulations and robustness checks with held-out-history, independent-embedding, and external LLM-based evaluators. These experiments show improvements in preference-aligned evidence selection under the evaluated feedback settings. A separate human-based study finds that scenario-conditioned summaries are preferred over generic (uniformly-aggregated) alternatives in 69.9\% of paired comparisons. 
\end{abstract}





\maketitle

\section{Introduction}
\label{sec:introduction}
Online product reviews are an important source of user-generated information on e-commerce websites. They influence purchase decisions \cite{filieri2015makes} by conveying customers' experiences with product quality, reliability, usability, and overall value for money \cite{mudambi2010makes}.
A larger number of reviews can increase confidence by providing more evidence about the product, but it can also make the decision-making process slower and harder \cite{hu2019enough}, especially when users are looking for specific information \cite{wang2025does}. 
Review summarization addresses this challenge by condensing the corpus into a shorter description \cite{jiang2021review,peal2022summarizing}. This idea has become especially relevant in practice with the introduction of AI-generated review summaries on platforms such as \amazon\cite{schermerhorn2023amazonai}. 


However, a concise product-level summary is not necessarily equally useful to every visitor \cite{jiang2021review,peal2022summarizing}.
Different users often care about different product attributes; for example, one may prioritize durability, another may prioritize value-for-money, and another may prioritize ease of use. 
These differences motivate personalized presentation of review information, a question also examined in prior research on consumer decision-making \cite{balan2021personalize}. This becomes even more challenging when preferences are not stationary \citep{javanmard2017perishability}. A user's interests may vary across products and may also evolve over time, requiring the system to adapt to preference drift rather than assume a fixed user-profile \citep{mahadik2020fast}. We study this challenge as a problem of \textit{personalized access to user-generated Web content}. Unlike traditional item recommendation, the product under consideration is given at each interaction: the system decides which evidence from that product's reviews to present, rather than which product to recommend. This leads to a central research question:  \emph{How can a system learn which evidence a user finds useful and adapt its summaries as that user's interests evolve?} We formulate personalized review summarization as an \textit{``online feedback-adaptive decision problem"}, in which the system generates preference-aligned summaries while learning which review evidence is most relevant to each user from lightweight scalar feedback.



\begin{figure*}
    \centering
    \includegraphics[width=0.88\linewidth, trim = {10pt 180pt 0 20pt},clip]{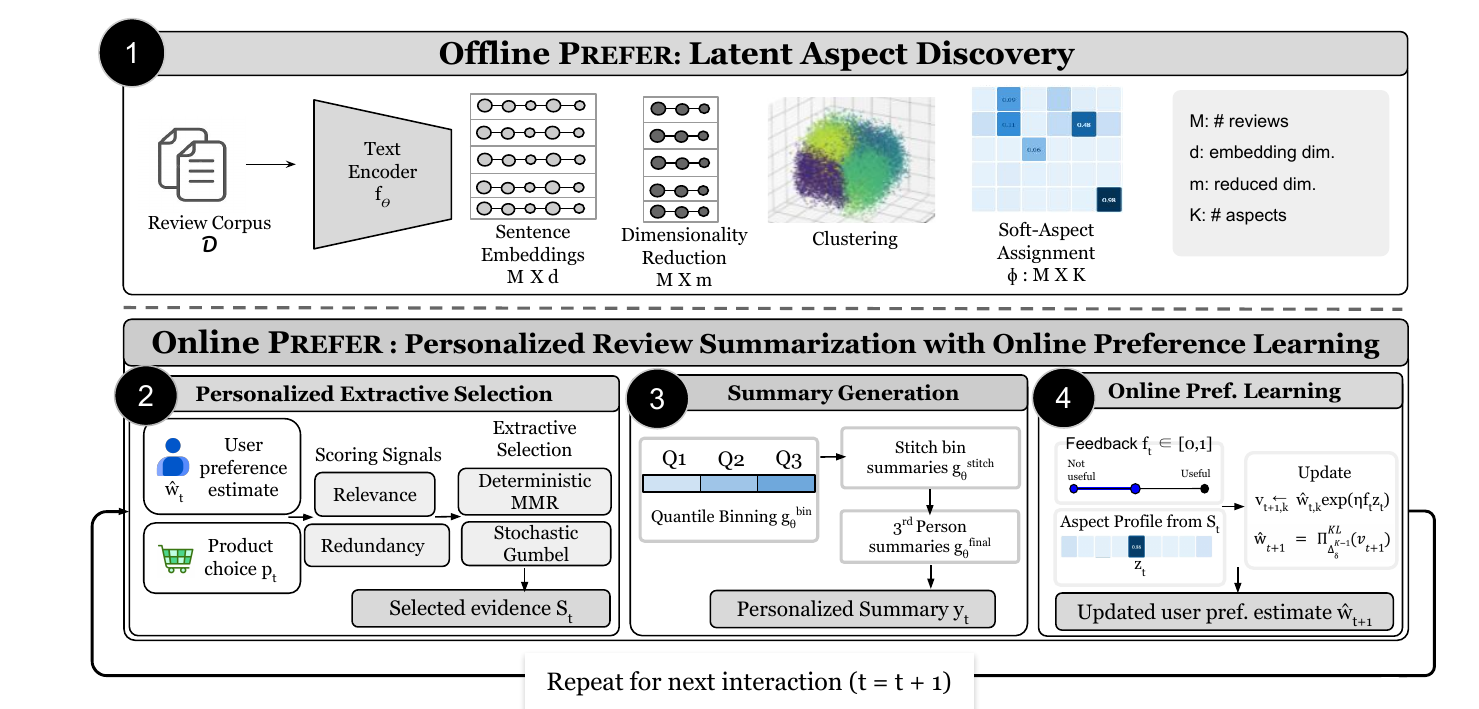}
    \vspace{-5pt}
    \caption{\textbf{Overview of \prefer.} Component (1) operates offline, while (2)-(4) run online. Summary generation in (3) uses an off-the-shelf LLM to rewrite the selected reviews into a coherent final summary. Section \ref{sec:PACER} details each component.}
    \label{fig:blockdiag}
    \vspace{-5pt}
\end{figure*}

We introduce \prefer
(as detailed in Fig. \ref{fig:blockdiag}), a modular framework for operationalizing this feedback-adaptive view. 
\prefer addresses the exploration-exploitation tradeoff when learning user preferences from feedback while generating preference-aligned summaries for the user. To this end, (1) it represents user interests and review content in a shared latent aspect space, (2) selects relevant and non-redundant evidence according to the current user preference estimate, and (4) updates this estimate using scalar feedback on generated summaries. The selected evidence is subsequently (3) rewritten into a coherent summary using an off-the-shelf language model. Thus, by separating \textit{evidence selection and preference learning} from \textit{language generation}, \prefer can adapt the information presented to users over repeated interactions without modifying or fine-tuning the underlying language model. 
Our contributions are threefold.
First, we formulate personalized review summarization as a sequential
evidence-selection problem with unknown, potentially evolving
aspect preferences and scalar feedback.
Second, we develop a modular framework combining
preference-conditioned evidence selection with entropic online
preference updates, and establish regret bounds for the resulting
feedback-dependent surrogate losses.
Third, we evaluate online adaptation through controlled simulations
and additional feedback settings with varying degrees of independence
from the learner's representation, complemented by a human study of
aspect interpretability and scenario-conditioned summary usefulness.

\section{Related Work}
\label{sec:relatedwork}
Research on review summarization dates back to \citet{hu2004mining}, who extract product features using part-of-speech patterns and group positive/negative opinion sentences around them. Subsequent works improve aspect identification through linguistic patterns and lexicons \citep[e.g.,][]{zhuang2006movie,dalal2013semisupervised}, aspect-expression clustering \citep{zhai2011clustering}, and topic-modeling with Bayesian approaches \citep{perikos2017system}.
Neural methods further reduced reliance on hand-crafted rules: \citet{he2017unsupervised} learn latent aspects directly from review text, while \citet{angelidis2018multiple} infer fine-grained segment-level sentiment from document labels.
Prior work also studies representative content selection once such sentiment representations are available: \citet{nguyen2013using} select review subsets that cover a collection of micro-reviews, while \citet{lappas2010efficient} formulate review selection as a combinatorial optimization problem. Most of these systems, however, \textit{produce product-level summaries}, including recent AI-generated review summaries on e-commerce platforms \citep{schermerhorn2023amazonai}, \textit{rather than personalizing content to an individual user's evolving preferences}.

Other studies have also examined whether personalizing review information improves consumer decision-making. \citet{balan2021personalize} compare raw reviews, attribute-level summaries, and preference-based review presentation, finding that personalization can improve decision-making effectiveness and efficiency. Our complementary human evaluation provides further evidence in this direction. Related work by \citet{wang2011collaborative} combines collaborative filtering with probabilistic topic modeling to learn interpretable user and item representations for scientific article recommendation. While these studies motivate preference-aware information access, \textit{they do not directly investigate how personalized review summaries can be iteratively adapted to changing user preferences over time}.

Prior work on interactive recommendation studies how user preferences can be learned and updated from feedback. \citet{deshpande2012linear} use linear bandits to model the exploration-exploitation trade-off between learning user preferences and making useful recommendations, while \citet{intayoad2020reinforcement,mahadik2020fast,xiang2010temporal} study evolving user interests, concept drift, and feedback-driven recommendation respectively. \citet{zhang2019text} incorporate natural-language feedback into reinforcement learning to account for users' historical preferences. Unlike \textit{these approaches}, which primarily focus on \textit{recommending items}, \prefer learns which review evidence to present for a given product. 

Recent work also explores adaptive personalization in  summarization and generation. \citet{gao2020preference} propose APRIL, combining active preference learning 
and reinforcement learning to adapt multi-document summaries from pairwise user feedback.
\citet{nieadaptive} model dynamic preference strengths to guide LLM generation, while \citet{chatterjee2026impersumm} incorporate action-aware user histories into personalized summarization. \citet{chatterjee2026beyond} model preference evolution through action-conditioned geometric walks and memory, and \citet{yu2026memweaver} construct hierarchical memories from textual interaction to personalize generation. While these approaches \textit{adapt summaries through pairwise preference elicitation} or \textit{use interaction histories to construct user representations that guide personalized content generation}, our focus is on learning which review evidence to present for a given product from scalar feedback across interactions. To this end, \prefer maintains a persistent aspect-level preference profile that is incrementally updated from scalar feedback on generated summaries and directly guides subsequent evidence selection, while keeping the underlying language generator frozen. 

\section{Problem Setup}
\label{sec:ProblemSetup}
We study personalized review summarization as a \textit{sequential decision problem} with unknown user preferences and structured actions. At each round $t$, user $u$ considers product $p_t$ with a corpus $\mathcal D_{p_t} = \{r_1,\dots,r_{n_{p_t}}\}$ of review sentences, where $n_{p_t}$ is the number of review sentences associated with the product. The system selects supporting sentences $S_t$, generates a summary $y_t$, and receives scalar feedback $f_t$ that informs its next decision. 

\textit{Products and their reviews.} We assume that these review sentences can be described by a small set of semantic themes, which we call \emph{aspects}. These aspects correspond to dimensions along which users often evaluate products (e.g., quality, ease of use, etc.), but they are \emph{latent} because they are inferred from review text rather than observed as fully labeled annotations. We denote this fixed set of latent review aspects by $\mathcal{A}=\{a_1,\dots,a_K\}$ where $K$ is the total number of aspects. Each sentence $r_i\in\mathcal D_{p_t}$ has (i) a semantic embedding $\mathbf s_i\in\mathbb R^d$ (where $d$ is the embedding dimension) which helps compare reviews semantically, (ii) an aspect-membership distribution vector $\boldsymbol\phi_i\in\Delta^{K-1}$ where $\Delta^{K-1}=\{\mathbf v\in\mathbb{R}_{\ge0}^K: \sum_{k=1}^K v_k=1\}$ is the simplex over the $K$ aspects and coordinate $\phi_{i,k}$ measures how strongly $r_i$ expresses aspect $k$, and  (iii) a token length $\ell_i$. 

\textit{Users and their preferences.} Each user $u$ has an unobserved, potentially time-varying aspect preference vector $\mathbf{w}_{u,t}\in\Delta^{K-1}$, whose $k$th coordinate represents the importance assigned to aspect $a_k$ at round $t$. The simplex representation captures how user interests are distributed across aspects. Preferences may remain stationary ($\mathbf w_{u,t}=\mathbf w_u$) or evolve across interactions. We quantify this evolution through the path variation
\vspace{-5pt}
\begin{equation}\label{eq:pathvariation}
V_{T_u}=\sum_{t=2}^{T_u}\|\mathbf w_{u,t}-\mathbf w_{u,t-1}\|_1,
\vspace{-2pt}
\end{equation}
where $V_{T_u}=0$ reflects stationary preferences. Since the true preferences are unknown, the system maintains an estimate $\widehat{\mathbf w}_{u,t}\in\Delta^{K-1}$, updated from interaction feedback. As detailed in Section~\ref{sec:omd-update}, the learner enforces a small positive lower bound on each estimated preference coordinate to ensure regularity of the online updates, without imposing the same restriction on the true user preferences.

\textit{Structured action space.} The system takes a \textit{structured action} comprising two stages: (i) \textit{combinatorial evidence selection}, choosing a nonempty subset of review sentences $S_t\in\mathcal F_{k,L}(p_t)$, where
\vspace{-5pt}
\begin{equation}\label{eq:feasible-evidence}
\mathcal F_{k,L}(p):= \left\{\varnothing\ne S\subseteq\mathcal D_p: |S|\le k,\ \sum_{i\in S}\ell_i\le L\right\},
\vspace{-3pt}
\end{equation}
with $k$ and $L$ denoting the cardinality and token budgets, respectively, and we assume at least one sentence is feasible; and (ii) \textit{compositional summary generation}, rewriting the selected evidence into a personalized summary $y_t=g_\theta(S_t,\widehat{\mathbf w}_{u,t},p_t)$, using an off-the-shelf language model. Here, $g_\theta$ denotes the language-model-based rewriting module. While both stages determine the displayed summary, our focus is on learning which review evidence to select rather than optimizing the underlying language model itself.
Additionally, we represent the aspect composition of the selected subset $S_t$ through an aggregate profile $\mathbf z_t:=\sum_{i\in S_t}\alpha_{t,i}\boldsymbol\phi_i \in\Delta^{K-1}$, where $\alpha_{t,i}\ge0$ and $\sum_{i\in S_t}\alpha_{t,i}=1$. The weights $\alpha_{t,i}$ capture the relative contribution of selected sentences, while $\mathbf z_t$ summarizes the aspects emphasized by the extracted evidence. 
The weighting schemes are detailed in Appendix~\ref{app:mmr-properties}.

\textit{Latent utility and feedback-driven surrogate loss.} We model the aspect-level utility of selected sentences as $U_t(\mathbf z_t;\mathbf w_{u,t}) :=\langle\mathbf w_{u,t},\mathbf z_t\rangle$, which measures its alignment with the user's preferences. However, neither $\mathbf w_{u,t}$ nor this utility is observed by the system, and hence, instead it scores candidate sentences using $\widehat U_t:=\langle\widehat{\mathbf w}_{u,t}, \mathbf z_t\rangle$. After generating and displaying summary $y_t$, the system observes scalar feedback $f_t\in[0,1]$, which serves as an indirect, potentially noisy signal of the human-perceived summary's usefulness.

To account for differences in users' typical feedback levels, we center the observed feedback around a baseline $b_t$, computed from past responses, and define $\widetilde f_t:=f_t-b_t$. Positive (negative) centered feedback indicates that the displayed summary performed above (below) this baseline. We incorporate this signal through the linear surrogate loss as follows:
\begin{equation}\label{eq:surrogateloss}
\ell_t(\mathbf v)
:=-\widetilde f_t\,\mathbf v^\top\mathbf z_t,
\qquad
\mathbf g_t:=\nabla\ell_t =-\widetilde f_t\mathbf z_t,
\end{equation}
for $\mathbf{v} \in \Delta^{K-1}$. In particular, $\ell_t(\widehat{\mathbf w}_{u,t}) =-\widetilde f_t\,\widehat U_t$. The loss is linear, and hence convex, on the simplex: minimizing it reinforces aspects emphasized by positively rated evidence and down-weights those associated with negative centered feedback. Thus, $(\widetilde f_t,\mathbf z_t)$ provides the learning signal for updating $\widehat{\mathbf w}_{u,t}$ before the next interaction. Importantly, this surrogate governs the \textit{online preference learning}, while the usefulness of selected sentences is evaluated separately through the latent aspect-level utility $U_t$.
Section~\ref{sec:omd-update} develops this preference-update mechanism and  establishes its regret guarantees.

\section{\prefer: Our Framework}
\label{sec:PACER}
Building on the \textit{online sequential decision} problem formulation in Section \ref{sec:ProblemSetup}, we now present our \prefer\ framework (Fig.~\ref{fig:blockdiag}), which comprises of two stages: (i) an \textit{offline representation stage} that discovers latent aspects from review text, and (ii) an \textit{online interaction stage} that selects preference-aligned evidence, generates summaries using an off-the-shelf language model, and refines user-preference estimates from scalar feedback. By separating preference learning and evidence selection from language generation, \prefer\ progressively adapts the information presented to users without modifying or fine-tuning the underlying language model.

\subsection{Offline Latent Aspect Discovery}
\label{sec:latent-aspect-discovery}
We first construct the semantic embeddings $\mathbf s_i$ and latent aspect representations $\boldsymbol\phi_i\in\Delta^{K-1}$ introduced in Section~\ref{sec:ProblemSetup}. By discovering aspects directly from review text, rather than relying on predefined attribute labels, this offline stage establishes the representation used for subsequent sentence selection and online preference learning. The construction proceeds in three steps:
\begin{enumerate}[label=(\roman*), noitemsep, topsep=0pt, leftmargin=*]
    \item \textit{Sentence embeddings:}
    We encode each review sentence $r_i$ using a pretrained sentence encoder $f_\theta:\mathcal R\to\mathbb R^d$, obtaining $\mathbf s_i=f_\theta(r_i)$, so that semantically related sentences have nearby representations. In our implementation, $f_\theta$ is \textit{sentence-transformers/all-MiniLM-L6-v2} \citep{reimers2019sentence,wang2020minilm}, a 6-layer encoder that maps text to $d=384$ dimensional embeddings; its contrastive \citep{gao2021simcse} training objective brings semantically related texts closer.
    
    \item \textit{Dimensionality reduction:}
    We then normalize the embeddings and apply principal component analysis (PCA) to retain the dominant directions of semantic variation: $\widetilde{\mathbf s}_i^{\mathrm{PCA}} =P_m(\mathbf s_i/\|\mathbf s_i\|-\bar{\mathbf s})$, where $P_m$ projects onto the top $m$ principal components and $\bar{\mathbf s}$ denotes the mean normalized embedding.

    \item \textit{Aspect discovery via Clustering:}
    We apply $K$-means clustering to the reduced embeddings, obtaining centroids $\{\mathbf c_k\}_{k=1}^K$, each representing a single latent aspect, capturing recurring semantic themes across review sentences.
\end{enumerate}


Hard cluster assignments restrict each sentence to a single aspect, although review sentences may simultaneously discuss multiple product attributes. We therefore represent each sentence through \textit{soft aspect memberships}, defined by a distance-based softmax over the cluster centroid: $\phi_{i,k} \propto \exp(-\tau\| \widetilde{\mathbf s}_i^{\mathrm{PCA}}-\mathbf c_k \|_2^2)$.
Here, $\tau>0$ controls assignment sharpness. Larger values of $\tau$ concentrate membership on the nearest aspects, while smaller values preserve mixed aspect memberships. This distinction is relevant to online preference learning, since the aspect representation determines how selected sentences contribute to the user-preference update.  To select $\tau$, we calibrate the ratio between the two largest memberships using the median gap between the two nearest cluster centroids. This gives $\tau^\star=\log(r)/\Delta_{\mathrm{med}}$, where $r>1$ is the target top-two membership ratio and $\Delta_{\mathrm{med}}>0$ is the median gap in squared centroid distances. The calibration details are provided in Appendix~\ref{app:tausoftaspectmembership}.

\subsection{Personalized Evidence (Sentence) Selection}
\label{sec:personalized-extractive-selection}

Given the latent aspect representations constructed $\boldsymbol{\phi}_i$ in Section~\ref{sec:latent-aspect-discovery}, \prefer now selects a \textit{subset} of review sentences $S_t$ that (i) reflects the user's current interests while (ii) avoiding redundant information. Since the true user preferences are unknown, each candidate sentence $r_i$ is scored using the current preference estimate $\widehat{\mathbf{w}}_{u,t}$. We define the \emph{estimated user-specific relevance score} by: $\mathrm{Rel}_{i,t}(u,p_t):=    \widehat{\mathbf{w}}_{u,t}^{\top}\boldsymbol{\phi}_i =    \sum_{k=1}^{K}\widehat{w}_{u,t,k}\phi_{i,k}$.
A higher score indicates stronger alignment between the sentence's aspect composition and the user's estimated preferences. To discourage selecting semantically overlapping sentences, we also measure redundancy using a nonnegative similarity kernel between two review sentence candidates as $\kappa(r_i,r_j)=\max\{0,\cos(\tilde{\mathbf{s}}_i^{\mathrm{PCA}},\tilde{\mathbf{s}}_j^{\mathrm{PCA}})\}$ (for notational brevity, we write $\kappa(i,j)=\max\{0,\cos(i,j)\}$), where $\cos(i,j):= \frac{\langle i, j\rangle} {\|i\|_2 \|j\|_2}$ denotes cosine similarity between the PCA-reduced sentence embeddings. A larger value of $\mathrm{cos}(i,j)$ indicates that the two sentences are semantically close, and therefore more likely to be redundant if selected together.
\newline

\noindent \textit{(1) Deterministic evidence selection.} We score the candidate set of sentences by trading off estimated user relevance against redundancy. For a parameter $\lambda\in[0,1]$, we define
\[
J_t(S;\widehat{\mathbf{w}}_{u,t}) = \lambda \sum_{i\in S}\mathrm{Rel}_{i,t} - (1-\lambda)\sum_{\substack{i,j\in S, i<j}}\kappa(i,j).
\]
The extractive selection problem at round $t$ is therefore: $S_t^\star \in \arg\max_{S} J_t(S;\widehat{\mathbf{w}}_{u,t})$. Now, for any $S\subseteq\mathcal D_{p_t}$ and any $j\notin S$, the true marginal gain of adding sentence $j$ is then $\Delta_t(j\mid S) = J_t(S\cup\{j\};\widehat{\mathbf{w}}_{u,t})-J_t(S;\widehat{\mathbf{w}}_{u,t}) = \lambda\,\mathrm{Rel}_{j,t}-(1-\lambda)\sum_{i\in S}\kappa(i,j)$.

In practice, however, evaluating the exact marginal gain requires summing similarities against all previously selected sentences at each step. To obtain a simpler extractive rule, we approximate this pairwise-sum by a maximal-redundancy penalty, yielding the classical Maximal Marginal Relevance (MMR) heuristic \citep{carbonell1998use}. Starting from $S_{t,0}=\varnothing$, at each extraction step $\tau=1,\dots,k$, we define the \textit{new marginal score} of a feasible candidate $j$ by
\begin{align}
a_{t,\tau}(j) := \lambda\,\mathrm{Rel}_{j,t} - (1-\lambda)\max_{i\in S_{t,\tau-1}}\kappa(i,j),
\label{eq:marginal-score}
\end{align}
where $\lambda\in[0,1]$ controls the relevance-redundancy tradeoff and the redundancy term is zero when $S_{t,\tau-1}=\varnothing$. The next selected sentence is then: $s_{t,\tau} \in \arg\max_{j\notin S_{t,\tau-1}} a_{t,\tau}(j)$, subject to the cardinality and token budgets in Eq.~\eqref{eq:feasible-evidence} and, we have $S_{t,\tau}=S_{t,\tau-1}\cup\{s_{t,\tau}\}$.

Eq.~\eqref{eq:marginal-score} is monotone non-increasing with respect to the selected set for each fixed candidate, as its maximum-redundancy penalty is monotone non-decreasing. This property supports efficient greedy extraction using cached redundancy scores; the details of which are provided in Appendix~\ref{app:mmr-properties}.
\newline

\noindent \textit{(2) Stochastic evidence selection.} A purely deterministic extractor always selects the highest-scoring sentences under the current preference estimate. This can be undesirable early in learning, when $\widehat{\mathbf{w}}_{u,t}$ may still be inaccurate. To address this, we introduce a stochastic variant based on Gumbel perturbations \citep{maddison2014sampling}. At extraction step $\tau$, let $\mathcal C_{t,\tau}:= \{j\notin S_{t,\tau-1}: S_{t,\tau-1}\cup\{j\}\in\mathcal F_{k,L}(p_t)\}$ denote the feasible candidates under Eq.~\eqref{eq:feasible-evidence}. For each candidate $j\in\mathcal C_{t,\tau}$, we perturb the marginal score in Eq.~\eqref{eq:marginal-score} by independent noise $g_{t,\tau,j}\sim\mathrm{Gumbel}(0,1)$ and select the next review sentence as $s_{t,\tau} =\arg\max_{j\in\mathcal C_{t,\tau}} \left\{ \beta_{\mathrm{ext},t}a_{t,\tau}(j) +g_{t,\tau,j} \right\},$ where $\beta_{\mathrm{ext},t}>0$ is an inverse-temperature parameter controlling the degree of exploration. Smaller values introduce greater randomness in evidence selection, whereas larger values increasingly favor candidates with higher marginal scores. To gradually favor exploitation over successive interactions, we adopt the schedule $\beta_{\mathrm{ext},t}= \min\{\beta_{\max},1+c_\beta\log(t+2)\}$, where $c_\beta>0$ is the growth rate and $\beta_{\max}>0$ caps the inverse temperature preventing the policy from becoming numerically unstable. By the Gumbel-Max identity \citep[Sec.~2, Eq.~(2)]{maddison2014sampling}, 
this selection rule induces a Boltzmann distribution over feasible candidates:
\begin{equation}\label{eq:gumbel-softmax}
\mathbb P(s_{t,\tau}=j\mid S_{t,\tau-1})
=
\frac{
\exp(\beta_{\mathrm{ext},t}a_{t,\tau}(j))
}{
\sum_{j'\in\mathcal C_{t,\tau}}
\exp(\beta_{\mathrm{ext},t}a_{t,\tau}(j'))
}.
\end{equation}

\subsection{Summary Generation}
\label{sec:contextual-abstractive-rewriting}

Given the selected sentences $S_t$, \prefer employs a lightweight \textit{hierarchical rewriting process} i.e. $y_t = g_\theta^{\mathrm{final}} (g_\theta^{\mathrm{stitch}}(g_\theta^{\mathrm{bin}}(S_t,\widehat{\mathbf{w}}_{u,t},p_t)))$, comprising of three stages: (i) \emph{Bin-level Compression} ($g_\theta^{\mathrm{bin}}$) : First, we assign each selected sentence in $S_t$ to its dominant latent aspect $ a_i = \arg \max_k \phi_{i,k}$. For each aspect $k$, we compute an empirical reviewer-level support score : $n_k = \left| \left\{ u_i:\ r_i\in S_t,\ a_i=k \right\}\right|$, where $u_i$ is the reviewer associated with sentence $r_i$. This counts the distinct reviewers whose contributed evidence's dominant aspect is $k$. We then group aspects into high-, mid-, and low-support bins using the $0.67$ and $0.33$ empirical quantiles of $\{n_k\}_{k=1}^K$, corresponding to themes mentioned by many, some, and few reviewers. (ii) \emph{Cross-Bin Stitching} ($g_\theta^{\mathrm{stitch}}$) :  We compress the evidence within each bin into an intermediate text block and ``stitch" the blocks into a coherent draft summary. (iii) \emph{Final Polishing} ($g_\theta^{\mathrm{final}}$) : Finally, we refine the draft into a third-person user-facing summary for publishing. Throughout all these stages, we use an off-the-shelf LLM, a prompted (see Appendix~\ref{app:rewriting-details}) API call to \textit{gemma-3-4b-it} or \textit{google/flan-t5-large}.

\subsection{Online Preference Learning from Feedback}
\label{sec:online-preference-adaptation}
After displaying the summary $y_t$, \prefer receives scalar feedback indicating its usefulness to user $u$. However, existing recommendation datasets do not provide feedback on generated summaries. Moreover, purchase decisions alone are imperfect proxies for summary usefulness, as users may browse or compare products without purchasing. We, therefore, begin with a controlled feedback model to study online adaptation and subsequently examine whether its benefits persist under more independent feedback mechanisms.

\subsubsection{Feedback Models}\label{sec:fdbk-models} We consider the following settings:

\textit{(1) Controlled feedback model.} To study how \prefer learns and adapts to user preferences over repeated interactions, we construct a controlled setting in which the latent preference vector $\mathbf w_{u,t}$ is ``\textit{known to the evaluator}" but ``\textit{hidden from the learner}". This preference may evolve over time to model user-interest drift. For each interaction, the evaluator uses this vector and the aggregated selected evidence profile $\mathbf z_t$ (as mentioned in Section~\ref{sec:ProblemSetup}) to construct a noisy utility signal $\tilde{U_t}=\mathbf w_{u,t}^{\top}\mathbf z_t+\epsilon_t$, where $\epsilon_t\sim\mathcal N(0,\sigma^2)$ captures variability in simulated user responses. We then generate bounded feedback as $f_t= {1}/({1+\exp[ -\gamma( \tilde{U_t}-\frac{1}{K} )]}),$ where $\gamma>0$ controls feedback sensitivity and $1/K$ is the utility level of a uniform aspect baseline. Importantly, the \textit{evaluator} uses the true $\mathbf w_{u,t}$ to generate feedback $f_t$ and assess preference alignment, while the \textit{learner} updates its preference estimate $\widehat{\mathbf{w}}_{u,t}$ using only the observed scalar feedback $f_t$ and selected-evidence profile $\mathbf z_t$.

\textit{(2) Independent Feedbacks beyond the shared representation.}
The above controlled model generates feedback using the same latent aspect representation $\mathcal{A}$ as the learner, potentially favoring its internal representation. To examine whether online \textit{learning} remains effective beyond this shared representation, we consider three additional \textit{evaluation} settings : (i) \textit{held-out history evaluator}, which constructs the \textit{learner's initial preference estimate} and the \textit{evaluator's preference target} from disjoint portions of user history while still retaining a shared aspect space, (ii) \textit{independent-embedding evaluator}, which evaluates selected sentences in a separately learned embedding and aspect space, and (iii) \textit{external LLM-simulated user}, which evaluates the displayed natural-language summary and its supporting evidence using an external LLM without accessing the learner's internal preference representation. Together, these settings test the robustness of feedback-driven adaptation to differences in how user preferences and feedback are represented. 

At each round, \prefer centers the observed feedback using $\widetilde f_t=f_t-b_t$ and constructs the surrogate gradient $\mathbf g_t=-\widetilde f_t\mathbf z_t$, as defined in Eq.~\eqref{eq:surrogateloss}. We consider two feedback baselines $b_t$: (i) a running mean $b_t^{\mathrm{mean}}= \frac{1}{t-1}\sum_{s=1}^{t-1}f_s$ for $t\ge2$, and (ii) an exponential moving average (EMA) $b_{t+1}^{\mathrm{ema}}= (1-\rho)b_t^{\mathrm{ema}}+\rho f_t$, where $\rho\in(0,1]$ controls responsiveness to recent feedback. We primarily use the EMA baseline in our experiments.

\subsubsection{Online Preference Updates via Mirror Descent}
\label{sec:omd-update}

We next update preference estimate using online mirror descent (OMD) \citep{beck2003mirror}. Since user preferences lie on the probability simplex, the update must preserve non-negativity and unit mass while incorporating feedback. We choose the negative-entropy distance-generating function $d(\mathbf w)=\sum_{k=1}^K w_k\log w_k$, which is $1$-strongly convex with respect to the $\ell_1$ norm (Lemma~\ref{lem:negative-entropy-dgf}, Appendix~\ref{app:onlinepreferenceadaption}). This property provides the curvature needed to control the effect of feedback gradients in the subsequent regret analysis. The corresponding Bregman divergence \citep[Eq.~(3.10)]{beck2003mirror} is $D_d(\mathbf v\|\mathbf w) :=d(\mathbf v)-d(\mathbf w) -\langle\nabla d(\mathbf w),\mathbf v-\mathbf w\rangle$, which reduces to Kullback-Leibler (KL) divergence for negative entropy (Proposition~\ref{prop:kl-bregman}, Appendix~\ref{app:onlinepreferenceadaption}). KL divergence measures changes in relative probability mass across aspects, making it well suited to learning
simplex-valued preferences.

To ensure that learned preference estimates remain bounded away from zero, we constrain the updates to the truncated simplex $\Delta_\delta^{K-1}:= \{\mathbf w\in\Delta^{K-1}:w_k\geq\delta,\ \forall k\}$, where $0<\delta<1/K$ is a prescribed minimum weight. This constraint applies only to the learned preference estimates and does not restrict the unknown true user preferences. We initialize $\widehat{\mathbf w}_{u,1}\in\Delta_\delta^{K-1}$, using a uniform profile unless otherwise specified. Using the feedback-driven gradient $\mathbf g_t$ from Eq.~\eqref{eq:surrogateloss}, the constrained entropic OMD update is
\begin{equation}
\label{eq:omd-update}
\widehat{\mathbf w}_{u,t+1}
= \arg\min_{\mathbf w\in\Delta_\delta^{K-1}} \left\{ \eta_t\mathbf g_t^\top\mathbf w + D_d(\mathbf w\|\widehat{\mathbf w}_{u,t}) \right\},
\end{equation}
where $\eta_t>0$ is the learning rate. The first term incorporates observed feedback, while the second regularizes changes in the preference estimate. By Proposition~\ref{prop:eg-update} (Appendix~\ref{app:onlinepreferenceadaption}), Eq.~\eqref{eq:omd-update} is implemented through a normalized multiplicative update followed by a KL projection:
\begin{equation}
\label{eq:exp-gradient-update}
\begin{aligned}
v_{t+1,k}= \frac{ \widehat w_{u,t,k} \exp(\eta_t\widetilde f_t z_{t,k}) }{ \sum_{j=1}^{K} \widehat w_{u,t,j} \exp(\eta_t\widetilde f_t z_{t,j}) }, \quad
\widehat{\mathbf w}_{u,t+1}
= \arg\min_{\mathbf w\in\Delta_\delta^{K-1}} D_d(\mathbf w\|\mathbf v_{t+1}).
\end{aligned}
\end{equation}
The projection admits the coordinate characterization $\widehat w_{u,t+1,k} =\max\{\delta,\xi_t v_{t+1,k}\}$, where $\xi_t>0$ is chosen such that the coordinates sum to one. It leaves the multiplicative update unchanged whenever $\mathbf v_{t+1}\in\Delta_\delta^{K-1}$. Consequently, every learned preference coordinate remains at least $\delta$, enforcing the learner-side interiority required for the subsequent analysis. Positive centered feedback increases the relative importance of aspects emphasized by the selected evidence before projection, while negative feedback decreases it. We use the decaying learning-rate schedule $\eta_t=\eta_0/\sqrt{1+c_\eta t}$, where $\eta_0,c_\eta>0$ control its initial magnitude and decay, respectively.
\newline

We analyze the update through the surrogate loss defined in Eq.~\eqref{eq:surrogateloss}. For the realized feedback and evidence sequence, define the dynamic surrogate regret $R_{T_u}^{\mathrm{OMD,dyn}} := \sum_{t=1}^{T_u} \left[ \ell_t(\widehat{\mathbf w}_{u,t}) - \ell_t(\mathbf w_{u,t}) \right]$, which compares the learned preference estimates against the latent preference sequence on the same observed losses. When preferences are stationary ($\mathbf w_{u,t}=\mathbf w_u$ for all $t$), this reduces to static regret against the fixed comparator $\mathbf w_u$. Building on standard regret guarantees for entropic OMD \citep{beck2003mirror} and time-varying comparators \citep{hall2013dynamical}, we derive a bound that accounts for both bounded feedback gradients and changes in the latent preference sequence, measured by the path variation \(V_{T_u}\) in Eq.~\eqref{eq:pathvariation}.

\begin{theorem}
\label{thm:omd-dynamic-regret}
Suppose $\mathbf z_t\in\Delta^{K-1}$, $|\widetilde f_t|\leq c$ with $c>0$, and $\widehat{\mathbf w}_{u,1}\in \Delta_\delta^{K-1}$ for some $0<\delta<1/K$. For any preference sequence $\mathbf w_{u,t}\in\Delta^{K-1}$, the constrained update in Eq.~\eqref{eq:omd-update}, with positive nonincreasing learning rates, satisfies
\begin{equation}
\label{eq:omd-dynamic-bound-general}
R_{T_u}^{\mathrm{OMD,dyn}} \leq \frac{ B_\delta+L_\delta(1-K\delta)V_{T_u} }{ \eta_{T_u} } + \frac{c^2}{2}\sum_{t=1}^{T_u}\eta_t + 2cK\delta T_u,
\end{equation}
where $B_\delta=\log(1/\delta)$ and $L_\delta=1+B_\delta$. In particular, for $\eta_t=\eta_0/\sqrt{1+c_\eta t}$, where $\eta_0,c_\eta>0$,
\begin{equation*}
\label{eq:omd-dynamic-bound-compact}
\begin{aligned}
R_{T_u}^{\mathrm{OMD,dyn}} \leq{}& \sqrt{1+c_\eta T_u} \left( \frac{ B_\delta+L_\delta(1-K\delta)V_{T_u} }{ \eta_0 } + \frac{c^2\eta_0}{c_\eta} \right) + 2cK\delta T_u.
\end{aligned}
\end{equation*}
If additionally $\mathbf w_{u,t}\in\Delta_\delta^{K-1}$ for every $t$, the same bound holds with $(1-K\delta)V_{T_u}$ replaced by $V_{T_u}$ and without the additive term $2cK\delta T_u$.
\end{theorem}

\noindent
The proof in Appendix~\ref{app:regret-analysis} uses a comparator-drift decomposition and KL-divergence bounds on the ``truncated simplex". To accommodate true preferences anywhere in the ``simplex", including its boundary, the analysis introduces an auxiliary smoothed comparator $\mathbf q_t=(1-K\delta)\mathbf w_{u,t} +\delta\mathbf 1$, which lies in $\Delta_\delta^{K-1}$, resulting in the additional approximation term ($2cK\delta T_u$). This construction is used only in the proof and does not alter the underlying user preferences or feedback process.



Theorem~\ref{thm:omd-dynamic-regret} establishes a general dynamic surrogate-regret bound that accounts for evolving preferences through $V_{T_u}$. For stationary preferences ($V_{T_u}=0$), it specializes to two cases (Corollary~\ref{cor:omd-static}, Appendix~\ref{app:regret-analysis}): \emph{(i) Preferences within $\Delta_\delta^{K-1}$:} For fixed $\delta,\eta_0,c_\eta$, and $c$, the average surrogate-regret upper bound decreases as $O(T_u^{-1/2})$. \emph{(ii) Preferences in the full simplex:} Similar to the evolving preferences case, a fixed $\delta>0$ introduces an additional $2cK\delta$ term in the average-regret upper bound. However, choosing $\delta=1/[K(T_u+1)]$ for a prescribed interaction horizon and holding it fixed throughout the run yields a vanishing $O(\log T_u/\sqrt{T_u})$ average-regret upper bound, including for true preferences on the simplex boundary.


While the above bounds characterize the surrogate-regret behavior of the online update relative to the true latent preference sequence, they do not establish how closely the learned preference estimates track these preferences. To substantiate the preference-learning component of personalization, we therefore complement the theoretical guarantees with empirical evaluations of preference tracking. Establishing a formal bound on the estimation error would require stronger assumptions linking observed feedback to latent utility and ensuring that selected evidence can distinguish different simplex-valued preferences. Accordingly, Section~\ref{sec:adaptive_preferences} compares learned and true preference trajectories, while Appendix~\ref{app:drift-alignment} provides complementary cosine-alignment diagnostics for both preference estimates and selected evidence.

\section{Case Study: on Amazon Reviews'23}
\label{sec:ourcasestudy}
Here, we evaluate \prefer\ on the \textsc{All Beauty} category of \amazon\ \reviews dataset~\citep{hou2024bridging}, containing 632K users and 112.6K products. Its diverse product attributes, including fragrance, skin compatibility, packaging, ingredients, etc provide a natural setting for heterogeneous user preferences. After deduplication and text cleaning, we construct review- and sentence-level datasets containing 583,190 reviews and 1,336,813 sentences, respectively.

\textit{Offline Aspect Discovery Setup.}
Following Section~\ref{sec:latent-aspect-discovery}, we compare latent aspect discovery at the review and sentence levels, with clustering diagnostics favoring $K=20$ and $K=10$ aspects, respectively. Fig.~\ref{fig:review_sentence_pca_compare} visualizes the resulting clusters projected onto the first three principal components, where each point represents a sampled (a) review or (b) sentence, and colors indicate latent aspect assignments. At the review level, clusters overlap considerably despite using $20$ aspects, as individual reviews often discuss several product attributes within the same text. In contrast, sentence-level clusters occupy more distinct and coherent regions, suggesting that sentences capture more focused semantic themes. We therefore use $K=10$ sentence-level aspects in the subsequent experiments.

\begin{figure}[t]
    \centering
    \begin{subfigure}{0.48\linewidth}
        \centering
        \includegraphics[width=\linewidth]{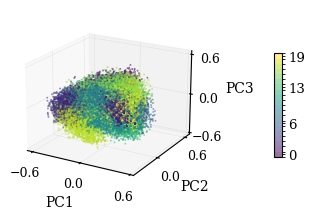}
        \caption{Review-level, $K=20$}
    \end{subfigure}
    \hfill
    \begin{subfigure}{0.48\linewidth}
        \centering
        \includegraphics[width=\linewidth]{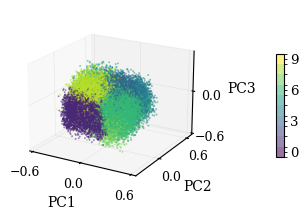}
        \caption{Sentence-level, $K=10$}
    \end{subfigure}
    \vspace{-10pt}
    \caption{Clustered PCA visualization for review-level and sentence-level aspect discovery.}
    \label{fig:review_sentence_pca_compare}
    \vspace{-10pt}
\end{figure}
\begin{figure}
    \centering
    \includegraphics[width=\linewidth, trim = {0 1pt 0 0}, clip]{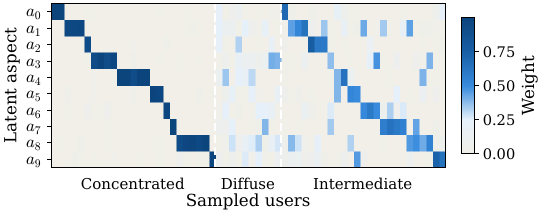}
    \vspace{-20pt}
    \caption{User-heterogeneity in  aspect preferences.}
    \label{fig:user_heterogeneity}
    \vspace{-15pt}
\end{figure}

\textit{Motivating user heterogeneity in aspect preferences.} 
Fig.~\ref{fig:user_heterogeneity} illustrates user heterogeneity in the discovered latent aspect space. For each user $u$, we construct an empirical aspect profile $\widehat{\mathbf w}_u=\frac{1}{|\mathcal I_u|}\sum_{i\in\mathcal I_u}\boldsymbol\phi_i$, where $\mathcal I_u$ denotes the user's review sentences. We group sampled users by the normalized entropy of their profiles, $\bar H(\widehat{\mathbf w}_u)=-\frac{1}{\log K}\sum_{k=1}^K\widehat w_{u,k}\log\widehat w_{u,k}$, distinguishing concentrated (bottom 20\%), diffuse (top 20\%), and intermediate (35th-65th percentile) profiles. The heatmap shows that users differ not only in which aspects they emphasize but also in how strongly their interests are concentrated across aspects. This variation motivates preference-conditioned evidence selection rather than a single generic product summary.

\section{Experiments and Results}
\label{sec:experiments}
We evaluate \prefer\ along four dimensions: (i) preference-aligned evidence selection and summary usefulness, (ii) online adaptation under evolving preferences, (iii) robustness to feedback generated outside the learner's internal representation, and (iv) comparison with alternative personalization methods. Unless otherwise stated, we use the \textsc{All Beauty} sentence-level aspect space with $K=10$ and the controlled feedback model in Section~\ref{sec:fdbk-models}.

\subsection{Personalization and Human Evaluation}
\label{sec:qualitative_goodness}
We first examine \textit{whether preference-conditioned extraction selects more relevant evidence than a generic (uniformly-aggregated) preference profile}. For 10 users, we learn preference profiles through 100 feedback rounds. We then freeze these profiles and evaluate all users on the same 50 held-out products, each containing at least 20 review sentences, across five random seeds, yielding $2{,}500$ matched user-product-seed evaluations. For each evaluation, we compare the selected evidence ($\mathbf z^{(\mathrm{pers})}_{u,p}$) under the learned preference profile $\widehat{\mathbf w}_u$ against that ($\mathbf z^{(\mathrm{gen})}_{p}$) obtained using a uniformly-aggregated profile $\widehat{\mathbf w}^{(\mathrm{gen})}=\frac{1}{K}\mathbf 1$. Alignment with the hidden true preference $\mathbf w_u$ is measured by cosine similarity. 

We distinguish preference-weighted utility, \(\mathbf w^\top\mathbf z\), from cosine evidence alignment, \(\cos(\mathbf w,\mathbf z)\). Utility rewards evidence on highly weighted aspects, whereas cosine alignment measures agreement with the overall preference profile. The two metrics need not induce the same ranking and are reported as complementary diagnostics.

Table~\ref{tab:cross_user_multi_product} shows that personalized extraction achieves, on average, a higher cosine evidence alignment  of $0.742$, compared with $0.513$, for uniformly-aggregated extraction. For each matched evaluation, we compute the difference between these two alignment scores; the mean paired gain is $+0.228$ (95\% CI: $[0.214,0.242]$). Personalized evidence also achieves higher alignment than uniformly-aggregated evidence in $76.6\%$ of the evaluations. Thus, learned preference profiles improve the alignment of selected evidence with the target user's interests across held-out products. 
 A qualitative example in Appendix~\ref{app:experiments_qual_bigtable} further shows how different user preferences lead to distinct summaries of the same product.
\newline
\begin{table}[t]
    \centering
    \caption{Personalized versus uniformly-aggregated evidence alignment with hidden true user preferences across $2{,}500$ matched user-product-seed evaluations (mean and 95\% CI).}
    \vspace{-10pt}
    \label{tab:cross_user_multi_product}
    \small
    \setlength{\tabcolsep}{5pt}
    \begin{tabular}{l c}
        \toprule
        \textbf{Metric Alignment} & \textbf{Mean [95\% CI]} \\
        \midrule
        Personalized: $\cos(\mathbf w_u,\mathbf z^{(\mathrm{pers})}_{u,p})$
            & $0.742\;[0.729,\,0.756]$ \\
        Uniformly-agg.:  $\cos(\mathbf w_u,\mathbf z^{(\mathrm{gen})}_{p})$
            & $0.513\;[0.500,\,0.527]$ \\
        \bottomrule
    \end{tabular}
    \vspace{-10pt}
\end{table}

\vspace{-8pt}
To assess the interpretability of the discovered latent aspects and the perceived usefulness of personalized summaries, we conduct a \textit{human study}\footnote{\href{https://prefer-appsurvey.vercel.app/}{Survey Link}; participants recruited through personal networks on a voluntary basis.} using \amazon \textsc{Movies \& TV} reviews. We collect responses from 61 participants, yielding $305$ sentence-intrusion judgments, $1{,}220$ sentence-to-aspect assignments, and $259$ paired summary comparisons. 

\textit{(1) Aspect coherence and interpretability.}
Participants identify the out-of-topic sentence among four review sentences and subsequently assign each sentence to one of two LLM-generated aspect labels or ``none of these.'' The intrusion response is recorded before revealing the labels. Participants correctly identify $71.8\%$ of intruders (compared to 25\% chance baseline) and achieve $85.1\%$ aspect-label matching accuracy (against 33.3\% chance baseline). These findings suggest that the discovered aspects capture semantic distinctions recognizable to human judges, although recognition varies across the evaluated aspect pairs (see Appendix~\ref{app:human-study}.).

\textit{(2) Scenario-specific summary usefulness.}
Given a \textit{viewing} scenario, participants rate personalized and generic summaries of the same movie/show on a five-point relevance scale and select the more useful one. \prefer\ is selected in $69.9\%$ of comparisons (181/259) and receives higher mean relevance ratings ($3.92$ versus $3.10$), with a mean paired improvement of $+0.815$ points.

Table \ref{tab:human_eval} reports 95\% confidence intervals (CIs), computed by resampling 61 participants with replacement, retaining all their responses, and recalculating each metric over 50,000 bootstrap repetitions. The 2.5th and 97.5th percentiles of these estimates define the CIs. Since the evaluated stimuli are fixed rather than independently sampled, these intervals capture participant-level variability conditional on the evaluated stimuli, without accounting for generalization to unseen aspect pairs, products, or scenarios.
Importantly, the human study evaluates the quality of the resulting aspect representations and summaries; the effectiveness of feedback-driven preference adaptation is assessed separately through controlled simulations in the following subsection.

\begin{table}[t]
\centering
\caption{Human evaluation results across 61 participants, reporting observed performance and 95\% bootstrap CIs.}
\vspace{-10pt}
\label{tab:human_eval}
\small
\setlength{\tabcolsep}{4pt}
\begin{tabular}{@{}lcc@{}}
\toprule
\textbf{Measure} & \textbf{Estimate} & \textbf{95\% CI} \\
\midrule
Sentence intrusion & $71.8\%$ & $[66.9,76.4]\%$ \\
Sentence-aspect matching & $85.1\%$ & $[82.4,87.7]\%$ \\
\prefer\ selected & $69.9\%$ & $[62.9,76.8]\%$ \\
\bottomrule
\end{tabular}
\end{table}

\subsection{Online Adaptation to Preference Drift}
\label{sec:adaptive_preferences}

\begin{table}[t]
\centering
\small
\caption{Effect of online preference updates under stationary preferences. Final-20-round normalized evidence utility (mean $\pm$ 95\% CI across five paired seeds).}
\vspace{-10pt}
\label{tab:online_static}
\setlength{\tabcolsep}{5pt}
\begin{tabular}{@{}lcc@{}}
\toprule
\textbf{Method}
& $\boldsymbol{\widetilde U}_{\mathrm{last}\text{-}20}$
& \textbf{Holm-adj. $p$} \\
\midrule
\textbf{\prefer-Gumbel}
& $\mathbf{0.3978\pm0.0125}$ & -- \\
Static-Gumbel
& $0.1007\pm0.0122$ & $1.23\times10^{-5}$ \\
\bottomrule
\end{tabular}
\vspace{-10pt}
\end{table}

We first assess online adaptation under stationary preferences. Table~\ref{tab:online_static} shows that \prefer-Gumbel achieves substantially higher final-20-round normalized evidence utility than Static-Gumbel ($0.3978$ vs.\ $0.1007$), highlighting the benefit of feedback-driven updates. 

We next examine whether \textit{these feedback-driven updates can track evolving user preferences}.
We construct a time-varying oracle preference $\mathbf w_{u,t}\in\Delta^{K-1}$ that shifts from aspect 6 to aspect 5 over a fixed drift window : $\mathbf w_{u,t} = (1-\rho_t)\mathbf w_{\mathrm{start}} + \rho_t \mathbf w_{\mathrm{end}}$, where  $\rho_t$ increases linearly from $0$ to $1$ during the drift window and remains constant outside it.
Fig.~\ref{fig:preference_drift_alignment} compares the learned preference estimates with the oracle trajectories. Both methods respond to the changing target with a lag, as feedback is observed only after each interaction. However, \prefer-Gumbel gradually reallocates preference mass toward Aspect 5 following the shift, whereas \prefer-MMR exhibits weaker post-drift recovery. This suggests that stochastic evidence selection can mitigate early evidence-selection lock-in and facilitate adaptation to evolving preferences. Additional diagnostics of  evidence alignment, surrogate regret under drift, additional ablations and sensitivity analysis are provided in Appendix~\ref{app:online-adaptation}.


\begin{figure}[t]
    \centering
    \includegraphics[width=0.5\textwidth]{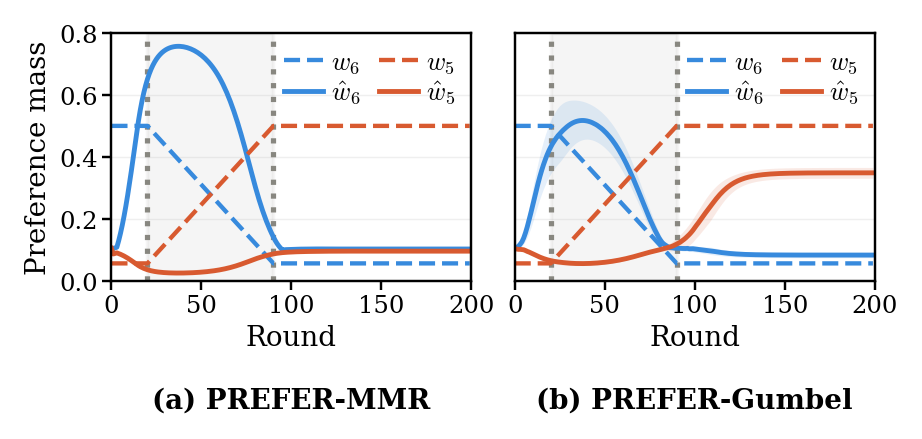}
    \vspace{-25pt}
    \caption{Tracking oracle (dashed) and learned
    (solid) preference mass for Aspects 6 and 5
    under preference drift. The shaded region
    indicates the transition window.}
    \label{fig:preference_drift_alignment}
    \vspace{-12pt}
\end{figure}


\subsection{Robustness to Independent Feedback}
\label{sec:independent_feedback}

We evaluate the three independent-feedback settings introduced in Section~\ref{sec:fdbk-models} \textit{to examine whether \prefer's online adaptation gains persist when feedback is decoupled from the learner's internal preference representation}. In each setting, we compare \prefer-Gumbel against a matched Static-Gumbel baseline that does not update its preference estimate from feedback. Table~\ref{tab:independent_feedback} reports the mean feedback received by \prefer\ and its feedback and utility gains over the static baseline, with 95\% confidence intervals.

\textit{(1) Held-out user history.}
We split one user's 156 historical review groups into disjoint fit and judge partitions. The learner is initialized from the fit partition, while a fixed preference estimated from the judge partition generates scalar feedback. Across five splits, five seeds, and 100 rounds, online adaptation improves mean feedback by $+0.00155$ and judge utility by $+0.00314$ over the static policy. The modest gains reflect the already high similarity between the two initial profiles (cosine similarity $0.9804$), leaving limited room for further adaptation.
\textit{(2) Independent embedding judge.}
We further separate the representations: \prefer\ learns using \textit{MiniLM}-based aspects, while an independently fitted \textit{MPNet}-based \cite{song2020mpnet} aspect space evaluates the selected raw sentences and generates scalar feedback. Across 25 paired split-seed runs of 100 rounds, online adaptation improves mean judge utility by $+0.0036$, demonstrating gains despite representation mismatch.
\textit{(3) External LLM-simulated users.} Finally, we replace latent-space feedback with judgments from \texttt{GPT-5-mini}, conditioned on natural-language user personas, displayed summaries, and supporting evidence, without access to the learner's internal state. Across five personas, three seeds, and 20 rounds, online \prefer\ improves both feedback and utility. The final-quarter gain reaches a high of $+0.0269$, with positive seed-averaged gains for all five personas. This experiment uses deterministic evidence summaries to isolate the feedback mechanism.

The absolute mean feedback scores in Table~\ref{tab:independent_feedback} reflect evaluator-specific scoring functions: the first two transform aspect-based utility into scalar feedback, whereas the LLM evaluates persona satisfaction, faithfulness, and decision usefulness. Hence, their absolute values are not directly comparable. Crucially, the positive online-over-static gains across all three settings indicate that \prefer's adaptation benefits are not confined to feedback generated within its own latent preference space.

\begin{table}[t]
\centering
\caption{Performance under independent feedback evaluators. Gains are online-minus-static differences with 95\% CIs, reported in units of $10^{-3}$. }
\vspace{-10pt}
\label{tab:independent_feedback}
\small
\setlength{\tabcolsep}{3pt}
\begin{tabular}{@{}lccc@{}}
\toprule
\textbf{Evaluator}
& \textbf{Mean $f_t$}
& $\boldsymbol{\Delta f}$ ($10^{-3}$)
& $\boldsymbol{\Delta U}$ ($10^{-3}$) \\
\midrule
Held-out history
& $0.97$
& $1.55\;[1.43,1.68]$
& $3.14\;[2.10,4.18]$ \\

Independent MPNet
& $0.76$
& $4.70\;[3.70,5.70]$
& $3.60\;[2.80,4.40]$ \\

External LLM
& $0.37$
& $22.30\;[1.90,49.10]$
& $23.60\;[2.70,52.00]$ \\
\bottomrule
\end{tabular}
\vspace{-3pt}
\end{table}

\subsection{Comparison with Alternative Baselines}
\label{sec:baseline-comparison}

\begin{table}[t]
\centering
\caption{Final-window feedback for the matched multi-user comparison. Values report mean $\pm$ SD across users after averaging matched seed repetitions within each user.}
\vspace{-6pt}
\label{tab:modern-baselines}
\small
\setlength{\tabcolsep}{4pt}
\begin{tabular}{@{}lcc@{}}
\toprule
\textbf{Method}
& \textbf{Neutral Init.}
& \textbf{History Init.} \\
\midrule

\textbf{\textsc{Prefer}-OMD}
& $\mathbf{0.701\pm0.109}$
& $\mathbf{0.697\pm0.097}$ \\

Semantic Online
& $0.480\pm0.053$
& $0.573\pm0.049$ \\

Thompson Sampling
& $0.696\pm0.097$
& $0.696\pm0.150$ \\

Long-Context Static
& $0.468\pm0.025$
& $0.463\pm0.042$ \\

Long-Context Online
& $0.474\pm0.028$
& $0.461\pm0.048$ \\

\bottomrule
\end{tabular}
\vspace{-6pt}
\end{table}

We compare \prefer-OMD against four alternative personalization methods under the controlled feedback model. All methods are evaluated using the same matched experimental design, comprising 3 users, 3 product-sequence seeds, 3 policy seeds, and 2 initialization settings, yielding 54 matched runs. Each run contains 100 interaction rounds. We consider both \textit{neutral} initialization, where no user-specific historical profile is provided initially, and \textit{history-informed} initialization, where each method is initialized using the corresponding user-history information.
(1) \textit{Semantic Online} maintains a dense semantic user profile derived from historical review embeddings, ranks candidate evidence by cosine similarity with redundancy filtering, and refines the profile through feedback-weighted embedding updates. This baseline contrasts \prefer's explicit latent-aspect preference representation and entropic OMD updates with direct embedding-space personalization. (2) \textit{Contextual Thompson Sampling} maintains a Bayesian posterior over latent aspect preferences. At each interaction, it samples a preference vector from this posterior to guide evidence selection and updates the posterior using observed feedback, providing an uncertainty-driven alternative to \prefer's mirror-descent update. (3) \textit{Long-Context Static} uses Gemma-3-4B to select relevant review evidence from the current candidate set, conditioned on the initial user information but without access to feedback from previous interactions. (4) \textit{Long-Context Online} additionally conditions the LLM on recent selected evidence and scalar feedback, enabling adaptation through \textit{in-context reasoning} rather than through an explicitly maintained preference vector. 

Table~\ref{tab:modern-baselines} reports feedback over the final 20\% of interaction rounds. It shows that the explicit aspect-space methods, \prefer-OMD and Thompson Sampling are the strongest methods across both initialization settings, with \prefer-OMD achieving the best overall final-window performance in the matched comparison. Semantic Online benefits more from history-informed initialization, while the long-context baselines remain comparatively weak and show little benefit from incorporating recent feedback into the prompt. Overall, the results favor explicitly maintained preference representations over direct semantic or in-context adaptation in this problem setting.

\section{Conclusion}
\label{sec:conclusion}

We introduced \prefer, a feedback-adaptive framework for personalized review summarization that learns user preferences from scalar feedback and selects preference-aligned, non-redundant evidence in a shared latent aspect space. We established surrogate-regret guarantees for online preference learning under evolving preferences. Experiments on \amazon \reviews demonstrate adaptation to preference drift, with improvements persisting across independent feedback settings. Human evaluations further support the interpretability of the discovered aspects and show that preference-conditioned summaries are preferred over generic alternatives. 

\bibliographystyle{ACM-Reference-Format}
\bibliography{references}

\appendix
\section{\prefer Components}
\subsection{Calibration of Hyperparameter $\tau$}
\label{app:tausoftaspectmembership}

Recall that the soft aspect memberships satisfy $\phi_{i,k}\propto\exp(-\tau d_{ik}^2)$, where $d_{ik}^2=\|\widetilde{\mathbf s}_i^{\mathrm{PCA}} -\mathbf c_k\|_2^2$. Let $d_{i(1)}^2\le d_{i(2)}^2$ denote the two smallest squared centroid distances and define $\Delta_i:=d_{i(2)}^2-d_{i(1)}^2$.

\begin{proposition}
\label{prop:soft-aspect-calibration}
For $K\ge2$, the ratio between the largest and second-largest memberships satisfies $\phi_{i(1)}/\phi_{i(2)}=\exp(\tau\Delta_i)$. Consequently, for a target ratio $r>1$ and $\Delta_{\mathrm{med}} :=\operatorname{median}_{i\in\mathcal I}\Delta_i>0$, the choice $\tau^\star={\log r}/{\Delta_{\mathrm{med}}}$ achieves a top-two membership ratio of $r$ at the median centroid gap.
\end{proposition}
\vspace{-15pt}
\begin{proof} $\phi_{i(1)}/{\phi_{i(2)}} =e^{-\tau d_{i(1)}^2}/{e^{-\tau d_{i(2)}^2}} =e^{\tau\Delta_i}$. Setting $e^{\tau^\star\Delta_{\mathrm{med}}}=r$ and taking log yields the result.
\end{proof}

\subsection{Properties and Implementation of MMR}
\label{app:mmr-properties}
Proposition~\ref{prop:mmr-monotone} below shows that later rounds of greedy extraction become progressively more conservative, which is exactly the intended redundancy-control effect of MMR. Moreover, using cached redundancy scores $m_{t,\tau}(j)=\max_{i\in S_{t,\tau}}\kappa(i,j)$, the MMR redundancy term can be updated after selecting $s_{t,\tau}$ by $m_{t,\tau}(j)=\max\{m_{t,\tau-1}(j),\kappa(s_{t,\tau},j)\}$, which reduces the greedy MMR extraction cost from $O(k^2 n_{p_t})$ to $O(k n_{p_t})$.

\begin{proposition}
\label{prop:mmr-monotone}
For a fixed candidate $j$, let $m(j\mid S):=\max_{i\in S}\kappa(i,j)$, with $m(j\mid\varnothing)=0$. If $A\subseteq B$, then $m(j\mid A)\leq m(j\mid B)$. Consequently, the MMR score satisfies $a_t(j\mid A)\geq a_t(j\mid B)$.
\end{proposition}
\vspace{-15pt}
\begin{proof}
Since $\kappa(i,j)\geq0$, enlarging the selected set cannot decrease the maximum similarity: $\max_{i\in A}\kappa(i,j)\leq \max_{i\in B}\kappa(i,j).$ The convention for the empty set preserves this inequality. Since the relevance term is fixed and $1-\lambda\geq0$, the MMR score is non-increasing.
\end{proof}



Additionally, recall from Section \ref{sec:ProblemSetup} that after selecting $S_t$, \prefer constructs the aggregate aspect profile $\mathbf z_t=\sum_{r_i\in S_t}\alpha_{t,i}\boldsymbol{\phi}_i$. The weights $\alpha_{t,i}$ can be uniform, $\alpha_{t,i}=1/m_t$, or can emphasize more important selected sentences. In particular, we consider  utility-based ($\alpha_{t,i}^{\mathrm{util}}\propto e^{\beta_\alpha a_{t,i}}$), rank-based ($\alpha_{t,i}^{\mathrm{rank}}\propto e^{-\gamma_\alpha(i-1)}$), and blended weights ($\alpha_{t,i}^{\mathrm{blend}}\propto e^{\beta_\alpha a_{t,i}-\gamma_\alpha(i-1)}$). Here, $a_{t,i}$ is the marginal score of the $i$th selected sentence, $\beta_\alpha$ controls emphasis on high-utility sentences, and $\gamma_\alpha$ controls emphasis on earlier-selected sentences.

\subsection{Prompts in the Summarization Module}
\label{app:rewriting-details}
\begin{preferprompt}{1}{Bin-level compression}
Summarize the following review evidence into a short,
concise text block while preserving the main content.
\end{preferprompt}

\begin{preferprompt}{2}{Cross-bin stitching}
You are given three evidence bins from product reviews:
HIGH, MID, and LOW. Your task is to stitch them into
one coherent, concise summary grounded in the review
evidence provided.

\textbf{Rules:}
\begin{promptsteps}
    \item Write exactly 3 brief paragraphs in total where (i) Paragraph 1 must describe the HIGH cluster. Prioritize it as the main takeaway. (ii) Paragraph 2 must describe the MID cluster. Use this bin to add supporting details and nuance.
    (iii) Paragraph 3 must describe the LOW cluster. Use it only for minor or less common preferences.
    \item Use phrases such as ``many users'', ``some users'', and ``a few users'' to reflect evidence strength.
\end{promptsteps}

\end{preferprompt}

\begin{preferprompt}{3}{Final polishing}
Rewrite the summary for user-facing display. Use third-person language; remove first-person expressions such as ``I'', ``me'',  and ``we''.
\end{preferprompt}

\subsection{Online Preference Adaptation}
\label{app:onlinepreferenceadaption}

The relative interior of the simplex is $\mathrm{ri}(\Delta^{K-1}) :=\{\mathbf w\in\Delta^{K-1}:w_k>0,\ \forall k\}$, consisting of vectors with strictly positive coordinates.

\begin{lemma}
\label{lem:negative-entropy-dgf}
The negative-entropy function $d$ is continuously differentiable on $\mathrm{ri}(\Delta^{K-1})$ and $1$-strongly convex on $\Delta^{K-1}$ with respect to $\|\cdot\|_1$, hence a valid distance-generating function.
\end{lemma}
\vspace{-12pt}
\begin{proof}
Differentiability follows from $\nabla d(\mathbf w)=\log\mathbf w+\mathbf 1$. Strong convexity follows from \citet[Proposition~5.1(a) and Remark~5.1]{beck2003mirror}.
\end{proof}

\begin{proposition}
\label{prop:kl-bregman}
For $\mathbf u,\mathbf w\in\mathrm{ri}(\Delta^{K-1})$, the induced Bregman divergence satisfies $D_d(\mathbf u\|\mathbf w) =\sum_{k=1}^K u_k\log(u_k/w_k) =\mathrm{KL}(\mathbf u\|\mathbf w)$.
\end{proposition}
\vspace{-12pt}
\begin{proof}
Substitute $\nabla d(\mathbf w)=\log\mathbf w+\mathbf 1$ into the Bregman definition \citep[Eq.~(3.10)]{beck2003mirror} and use $\sum_k u_k=\sum_k w_k=1$.
\end{proof}

\begin{proposition}[Constrained entropic OMD update]
\label{prop:eg-update}
For $\widehat{\mathbf w}_{u,t} \in\Delta_\delta^{K-1}$, finite $\mathbf g_t$, and $\eta_t>0$, the constrained entropic OMD update in Eq.~\eqref{eq:omd-update} has the unique solution characterized by Eq.~\eqref{eq:exp-gradient-update}. In particular, letting $v_{t+1,k} \propto \widehat w_{u,t,k}\exp(-\eta_tg_{t,k})$, the updated coordinates satisfy $\widehat w_{u,t+1,k} = \max\{\delta,\xi_t v_{t+1,k}\}$, where $\xi_t>0$ is uniquely determined by $\sum_k\widehat w_{u,t+1,k}=1$.
\end{proposition}
\vspace{-12pt}
\begin{proof}
By the entropic mirror-descent formulation \citep[Section~5]{beck2003mirror}, letting $Z_t=\sum_j\widehat w_{u,t,j}\exp(-\eta_tg_{t,j})$ gives $\eta_t\langle\mathbf g_t,\mathbf w\rangle +D_d(\mathbf w\|\widehat{\mathbf w}_{u,t}) =D_d(\mathbf w\|\mathbf v_{t+1})-\log Z_t.$ Thus, Eq.~\eqref{eq:omd-update} is equivalent to the Bregman (KL) projection of $\mathbf v_{t+1}$ onto $\Delta_\delta^{K-1}$. By strict convexity and the KKT conditions, the unique minimizer satisfies $\widehat w_{u,t+1,k} =\max\{\delta,\xi_t v_{t+1,k}\}$, where $\xi_t>0$ is uniquely determined by the unit-sum constraint.

\end{proof}

\subsubsection{Regret Analysis of Online Mirror Descent}
\label{app:regret-analysis}

Fix a user $u$ and write $T=T_u$, $\widehat{\mathbf w}_t =\widehat{\mathbf w}_{u,t}$, $\mathbf w_t=\mathbf w_{u,t}$, $R_T=R_{T_u}^{\mathrm{OMD,dyn}}$, and $V_T=V_{T_u}$. Throughout, $D_d$ denotes the KL divergence induced by negative entropy, and we define $B_\delta:=\log(1/\delta)$ and $L_\delta:=1+B_\delta$.

\begin{lemma}[Bounded feedback gradients]
\label{lem:gradient-bound}
If $\mathbf z_t\in\Delta^{K-1}$ and $|\widetilde f_t|\le c$, then $\|\mathbf g_t\|_\infty\le c$.
\end{lemma}
\vspace{-12pt}
\begin{proof}
Since $\mathbf g_t=-\widetilde f_t\mathbf z_t$, $\|\mathbf g_t\|_\infty =|\widetilde f_t|\|\mathbf z_t\|_\infty\le c$.
\end{proof}

\begin{lemma}[One-step entropic OMD inequality]
\label{lem:omd-one-step}
For any $\mathbf q\in\Delta_\delta^{K-1}$, the update in Eq.~\eqref{eq:omd-update} satisfies $\eta_t\langle\mathbf g_t,\widehat{\mathbf w}_t-\mathbf q\rangle \le D_d(\mathbf q\|\widehat{\mathbf w}_t) -D_d(\mathbf q\|\widehat{\mathbf w}_{t+1})+\frac{\eta_t^2}{2}\|\mathbf g_t\|_\infty^2$.
\end{lemma}
\vspace{-9pt}
\begin{proof}
Apply the optimality argument leading to \citet[Eq.~(4.21)]{beck2003mirror} to the linear loss $\ell_t$ and comparator $\mathbf q$, using the $\ell_1/\ell_\infty$ norm pair and the strong-convexity constant $1$ \citep[Proposition~5.1(a)]{beck2003mirror}.
\end{proof}

\begin{lemma}[KL drift control and Bound]
\label{lem:entropy-bregman-lipschitz}
For $\mathbf u,\mathbf v,\mathbf x \in\Delta_\delta^{K-1}$, the KL drift is controlled by $|D_d(\mathbf u\|\mathbf x)-D_d(\mathbf v\|\mathbf x)| \le L_\delta\|\mathbf u-\mathbf v\|_1$. Additionally, KL can be bounded by $D_d(\mathbf u\|\mathbf v)\le B_\delta$.
\end{lemma}
\vspace{-12pt}
\begin{proof}
For $F(\mathbf v):=D_d(\mathbf v\|\mathbf x)$, the Bregman definition gives $\nabla F(\mathbf v)=\nabla d(\mathbf v)-\nabla d(\mathbf x) =\log(\mathbf v/\mathbf x)$. Its $\ell_\infty$ norm is at most $B_\delta\le L_\delta$ on $\Delta_\delta^{K-1}$. Integrating along the segment between $\mathbf u$ and $\mathbf v$ gives the result. Since $\sum_k u_k\log u_k\le0$ and $v_k\ge\delta$, we have  $D_d(\mathbf u\|\mathbf v) \le-\sum_k u_k\log v_k\le B_\delta$.
\end{proof}

\begin{proof}[Proof of Theorem~\ref{thm:omd-dynamic-regret}]
Following the comparator-drift argument of \citet[Lemma~5 and proof of Theorem~4]{hall2013dynamical}, we specialize the analysis to entropic OMD with identity dynamics and the KL-specific constants established above.

For arbitrary $\mathbf w_t\in\Delta^{K-1}$, define $\mathbf q_t=(1-K\delta)\mathbf w_t+\delta\mathbf 1$. Then $\mathbf q_t\in\Delta_\delta^{K-1}$ and $\sum_{t=2}^T\|\mathbf q_t-\mathbf q_{t-1}\|_1 =(1-K\delta)V_T$.

Applying Lemma~\ref{lem:omd-one-step} with comparator $\mathbf q_t$, summing over $t$, and using Lemmas~\ref{lem:gradient-bound} and \ref{lem:entropy-bregman-lipschitz} gives, by the weighted telescoping argument in \citet[proof of Theorem~4]{hall2013dynamical},
\begin{equation}
\begin{aligned}
\sum_{t=1}^T
[\ell_t(\widehat{\mathbf w}_t)-\ell_t(\mathbf q_t)]
\leq
\frac{B_\delta+L_\delta(1-K\delta)V_T}{\eta_T}
+\frac{c^2}{2}\sum_{t=1}^T\eta_t.
\end{aligned}
\label{eq:omd-smoothed-comparator-bound}
\end{equation}
Here, $D_d\leq B_\delta$ and the KL drift is bounded by $L_\delta$ times the comparator variation; nonincreasing step sizes imply $1/\eta_t\leq1/\eta_T$. Moreover $\|\mathbf q_t-\mathbf w_t\|_1\leq2K\delta$, so Lemma~\ref{lem:gradient-bound} yields $\ell_t(\mathbf q_t)-\ell_t(\mathbf w_t) \leq \|\mathbf g_t\|_\infty \|\mathbf q_t-\mathbf w_t\|_1 \leq2cK\delta.$ Adding this comparison error to Eq.~\eqref{eq:omd-smoothed-comparator-bound} establishes Eq.~\eqref{eq:omd-dynamic-bound-general}.

Finally, for $\eta_t=\eta_0/\sqrt{1+c_\eta t}$, we have $\eta_T^{-1}=\sqrt{1+c_\eta T}/\eta_0$ and $\sum_{t=1}^T\eta_t \leq(2\eta_0/c_\eta)\sqrt{1+c_\eta T}$, which yields the second Eq. in the Theorem. If $\mathbf w_t\in\Delta_\delta^{K-1}$, taking $\mathbf q_t=\mathbf w_t$ directly removes the approximation term and gives the sharper stated bound.
\end{proof}

\begin{corollary}[Stationary preferences]
\label{cor:omd-static}
Under the assumptions of Theorem~\ref{thm:omd-dynamic-regret}, suppose $\mathbf w_t=\mathbf w_u$ for all $t$. If $\mathbf w_u\in\Delta_\delta^{K-1}$, then
\[
\begin{aligned}
R_T^{\mathrm{stat}}(\mathbf w_u)
\leq
\sqrt{1+c_\eta T}
\left(
{B_\delta}/{\eta_0}
+{(c^2\eta_0)}/{c_\eta}
\right),
\end{aligned}
\]
where $R_T^{\mathrm{stat}}(\mathbf w_u) := \sum_{t=1}^{T} \left[ \ell_t(\widehat{\mathbf w}_t) -\ell_t(\mathbf w_u) \right].$ For fixed $c,c_\eta,\delta,\eta_0$, the average surrogate-regret upper bound is $O(T^{-1/2})$.

For an arbitrary stationary comparator $\mathbf w_u\in\Delta^{K-1}$, the same cumulative upper bound holds with an additional term $2cK\delta T$.
\end{corollary}
\vspace{-11pt}
\begin{proof}
Setting $V_T=0$ in the interior-comparator case of Theorem~\ref{thm:omd-dynamic-regret} gives the first bound. For arbitrary stationary comparators, apply the general bound with $V_T=0$.
\end{proof}

\paragraph{Vanishing regret bound over the full simplex.}
For a prescribed horizon $T$, choosing $\delta=1/[K(T+1)]$ and holding it fixed throughout the run gives $B_\delta=\log(K(T+1))$ and $2cK\delta T\leq 2c$. Consequently, for fixed $\eta_0,c_\eta,c$, and $K$, Theorem~\ref{thm:omd-dynamic-regret} implies $\frac{R_T}{T} \leq O\left( \frac{(1+V_T)\log(K(T+1))}{\sqrt{T}} \right) +O(T^{-1})$. Thus, for stationary preferences ($V_T=0$), the average surrogate-regret upper bound vanishes at rate $O(\log T/\sqrt{T})$, even when the true preferences lie on the simplex boundary.



\paragraph{Scope of the guarantee.}
The regret compares estimated preference vectors against the latent preference sequence on the realized surrogate losses. Because the latent comparator need not minimize these losses, regret can be negative. A vanishing one-sided average upper bound implies $\limsup_{T\to\infty}R_T/T\leq0$, not necessarily $R_T/T\to0$. The analysis does not establish recovery of the true preferences, regret against counterfactual evidence-selection policies, or optimality of the generated summaries.

\section{Additional Experiments}
\label{app:experiments}

\subsection{Personalization and Human Study}
\subsubsection{Qualitative Illustration of Personalization}
\label{app:experiments_qual_bigtable}

\begin{table*}[t]
\centering
\caption{Personalized summaries of the same product under different input aspect-preference profiles. $\cos(\mathbf w,\mathbf z)$ measures alignment between the input profile $\mathbf w$ and selected-evidence aspect distribution $\mathbf z$.}
\vspace{-8pt}
\label{tab:cross_user_heterogeneity}
\small
\begin{tabularx}{\textwidth}{
@{}
p{0.16\textwidth}
p{0.14\textwidth}
p{0.12\textwidth}
X
@{}
}
\toprule
\textbf{Profile} &
\textbf{${\mathbf w}$} &
\textbf{$\cos(\mathbf w,\mathbf z)$} &
\textbf{Dominant themes} \\
\midrule

Aspect 0 &
$\mathbf e_0$ &
0.9995 &
Visual presentation, color variety, and giftability. \\

\multicolumn{4}{@{}p{\textwidth}@{}}{
\textit{Generated summary.}
\emph{
Many users appreciate the product's color variety, particularly bright options such as pink and blue. Some users highlight attractive colors, convenient size, fast drying, and consistent texture. A few users prefer natural colors or appreciate the multi-color packaging for gifting.
}
} \\

\midrule

Aspect 2 &
$\mathbf e_2$ &
0.9914 &
Exfoliation effectiveness and suitability for body use. \\

\multicolumn{4}{@{}p{\textwidth}@{}}{
\textit{Generated summary.}
\emph{
Many users say the product helps improve the appearance and feel of skin, including concerns such as visible scars, rough texture, and Keratosis Pilaris. Some users value the product for softness, with reports of long-term use and compliments on smoother-looking skin. However, a few users suggest it is better suited for body exfoliation on tougher areas such as arms, back, or feet than for delicate facial use, since the texture may feel too rough for sensitive areas.
}
} \\

\midrule

Mixed aspects &
$\frac{1}{2}\mathbf e_0+\frac{1}{2}\mathbf e_2$ &
0.7527 &
Color variety, presentation, and exfoliating properties. \\

\multicolumn{4}{@{}p{\textwidth}@{}}{
\textit{Generated summary.}
\emph{
Many users emphasize the product's attractive colors, fast drying, and the expected rough exfoliating texture. Some users also express interest in darker colors or more neutral variants. They mention that different colors do not appear to change the product's roughness, while appreciating the availability of multiple colors, and the option to use the set as gifts. A few users also prefer a natural or undyed version and also mention skin appearance outcomes, such as more even tone, reduced visibility of scars, and improved confidence.
}
} \\

\midrule

Uniformly-Aggregated &
$\frac{1}{K}\mathbf 1$ &
0.9897 &
General product usability. \\

\multicolumn{4}{@{}p{\textwidth}@{}}{
\textit{Generated summary.}
\emph{
Many users find the product useful for body scrubbing, shaving preparation, and reaching the back during showering. A few users describe it as a strong general-purpose washcloth or gift item, while also noting practical usage cautions such as pairing it with appropriate soap to avoid clogged follicles.
}
} \\

\bottomrule
\end{tabularx}
\end{table*}
We illustrate how different user preference profiles lead to distinct summaries of the same product, namely \emph{Salux Nylon Japanese Beauty Skin Bath Wash Cloth/Towel} from the \textsc{ALL Beauty} category of \amazon \reviews dataset. Fixing its review corpus, we consider three user preference profiles: Aspect 0 ($\mathbf e_0$), Aspect 2 ($\mathbf e_2$), and an equal mixture of both ($\frac{1}{2}\mathbf e_0+\frac{1}{2}\mathbf e_2$), where $\mathbf e_k$ denotes the $k$th standard basis vector. We also include the uniformly-aggregated profile $\frac{1}{K}\mathbf 1$. The aspects are discovered from review embeddings rather than predefined product attributes. Table~\ref{tab:cross_user_heterogeneity} presents the resulting summaries and their cosine alignment $\cos(\mathbf w,\mathbf z)$ with the respective input profiles. Aspect 0 emphasizes visual presentation and color variety, whereas Aspect 2 emphasizes exfoliation effectiveness and suitability for body use. The mixed profile combines these themes, while the uniformly-aggregated profile provides a broader overview of product usability. This illustrates how preference-conditioned evidence selection changes the information presented to users without changing the underlying review corpus. The rewriting prompt is provided in Appendix~\ref{app:rewriting-details}.

\subsubsection{Variation in aspect recognizability in Human Study}
\label{app:human-study}

\begin{table}[t]
\centering
\caption{Human intrusion accuracy and geometric similarity for four selected aspect pairs.}
\vspace{-8pt}
\label{tab:aspect_similarity_human}
\small
\setlength{\tabcolsep}{3.5pt}
\begin{tabular}{@{}llcc@{}}
\toprule
\textbf{Target aspect} &
\textbf{Intruder aspect} &
\textbf{Accuracy} &
\textbf{Cosine sim.} \\
\midrule
Music & Humor
    & $96.8\%$ & $-0.011$ \\
Family \& kids & Story \& plot
    & $72.6\%$ & $-0.072$ \\
TV episodes & Overall praise
    & $72.6\%$ & $0.085$ \\
Ending & General opinion
    & $64.5\%$ & $0.205$ \\
\bottomrule
\end{tabular}
\end{table}

We examine the relationship between geometric similarity and human recognition accuracy for four selected aspect pairs in the \textsc{Movies \& TV} category. Table~\ref{tab:aspect_similarity_human} reports cosine similarity between their merged-aspect centroids alongside sentence-intrusion accuracy. Music-Humor exhibits near-zero centroid similarity ($-0.011$) and the highest recognition accuracy ($96.8\%$), whereas Ending-General opinion has higher similarity ($0.205$) and lower accuracy ($64.5\%$). These observations are consistent with the possibility that semantic overlap makes aspect distinctions more difficult to recognize.

\subsection{Online Adaptation Experiments}
\label{app:online-adaptation}

\subsubsection{Component Ablations}
\label{sec:ablation}

We further examine extraction and weighting components of \prefer-Gumbel over $T_u=100$ rounds and five paired seeds, using $K=10$ and calibrated temperature $\tau^\star$. All variants share the same initialization, target preferences, and product sequences. We report normalized evidence utility $\widetilde U_t =(\mathbf w_u^\top\mathbf z_t-1/K)/(1-1/K)$, averaged over the final 20 rounds. Table~\ref{tab:gumbel_ablation} shows that uniform-random extraction slightly reduces utility, whereas near-deterministic extraction performs marginally better ($0.4001$). Uniform evidence weighting has negligible effect. Together with Table~\ref{tab:online_static}, these results indicate that online preference updates, rather than a particular extraction or weighting choice, drive the largest gains in this stationary setting.

\begin{table}[t]
    \centering
    \small
    \caption{Component ablations (mean $\pm$ 95\% CI across 5 seeds). $p$-values are paired comparisons against the full method.}
    \vspace{-10pt}
    \label{tab:gumbel_ablation}
    \resizebox{\columnwidth}{!}{
    \begin{tabular}{l l c c}
        \toprule
        \textbf{Variant} &
        \textbf{Component changed} &
        ${\widetilde U}_{\mathrm{last}\text{-}20}$ &
        \textbf{Holm-adj. $p$} \\
        \midrule


        Uniform-random
        & $\beta_{\mathrm{ext},t}=0$
        & $0.3685 \pm 0.0229$
        & $0.078$ \\

        Near-deterministic
        & $\beta_{\mathrm{ext},t}\gg 1$
        & $0.4001 \pm 0.0127$
        & $0.012$ \\

        Uniform $\alpha_{t,i}$
        & $\alpha_{t,i}=1/|S_t|$
        & $0.3980 \pm 0.0128$
        & $0.502$ \\
        \bottomrule
    \end{tabular}
    }
\end{table}

\subsubsection{Evidence Alignment under Drift}
\label{app:drift-alignment}

To complement the aspect-level preference trajectories in Fig.~\ref{fig:preference_drift_alignment}, we measure preference-estimate alignment $A_t^{\mathrm{pref}} =\cos(\mathbf w_{u,t},\widehat{\mathbf w}_{u,t})$ and selected-evidence alignment $A_t^{\mathrm{evid}} =\cos(\mathbf w_{u,t},\mathbf z_t)$, both relative to the current oracle preference. Fig.~\ref{fig:preference_drift_alignmentAt} shows that both alignment measures decline during the preference transition. Following the drift, \prefer-Gumbel recovers both preference and evidence alignment, whereas \prefer-MMR exhibits weaker recovery, particularly in selected-evidence alignment. This confirms that the differences observed in aspect-level tracking also affect the evidence selected for summarization.


\begin{figure}[t]
    \centering
    \includegraphics[width=\columnwidth, trim={0 25pt 0 25pt}, clip]{
        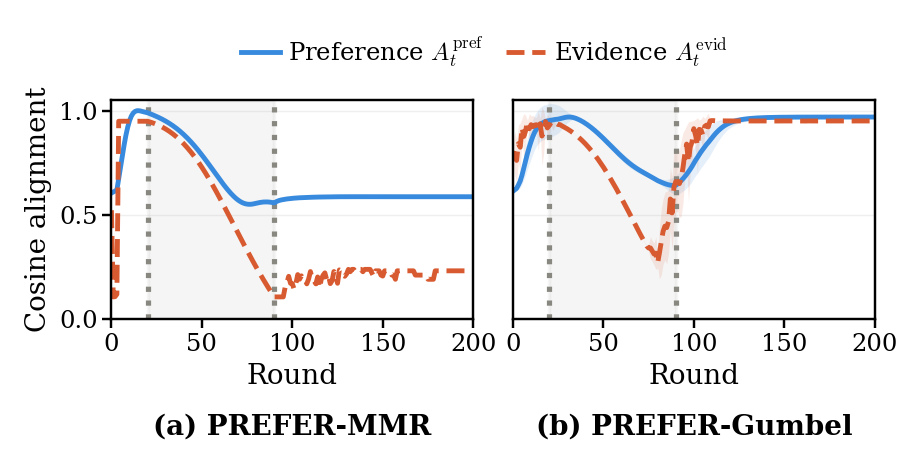}
    \vspace{-18pt}
    \caption{Preference-estimate and selected-evidence
    alignment with the drifting oracle preference
    $\mathbf w_{u,t}$.}
    \label{fig:preference_drift_alignmentAt}
    \vspace{-10pt}
\end{figure}

\subsubsection{Surrogate-Regret Diagnostics}
\label{app:surrogate-regret}


\begin{figure}[t]
    \centering
    \includegraphics[width=0.85\columnwidth, trim = {0 20pt 0 25pt}, clip]{
        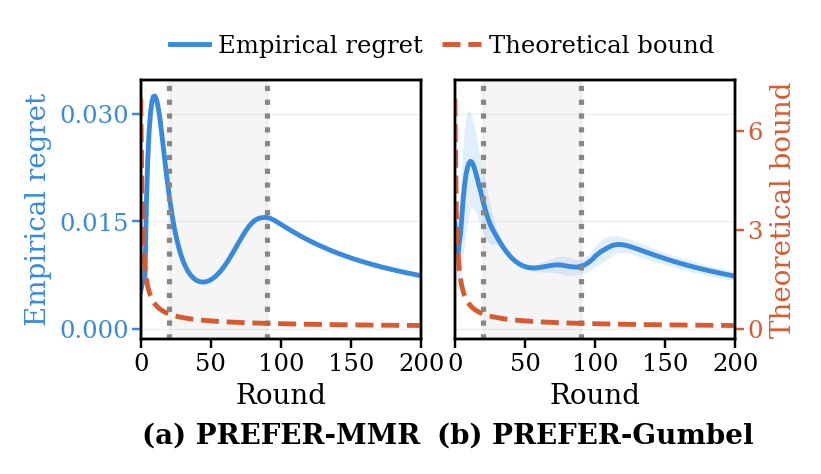}
        \vspace{-8pt}
    \caption{Average surrogate regret under preference drift compared with the theoretical upper bounds which has a separate vertical scale. Shaded bands indicate variability across seeds, and vertical lines mark the drift window.}
    \label{fig:drift_avg_regret}
    \vspace{-8pt}
\end{figure}

Under preference drift, Fig.~\ref{fig:drift_avg_regret} shows an increase in average surrogate regret near the preference transition, followed by a decline as additional feedback accumulates. This behavior illustrates the effect of a changing comparator, whose path variation appears explicitly in the dynamic-regret bound of Theorem~\ref{thm:omd-dynamic-regret}. The theoretical bounds, however, are substantially looser than the observed regret.

\subsubsection{Sensitivity to Aspect Dimensionality}
\label{sec:k_sensitivity}

We vary $K\in$ \{5,10,15, 20,25\}, refitting the latent aspect representation and recalibrating the assignment temperature for each $K$. We run 100 interaction rounds across five paired random seeds and measure the online-over-static gain $\Delta$ in mean cosine evidence alignment over the final 20 rounds. As Table~\ref{tab:k_sensitivity} shows, \prefer-Gumbel improves over its static counterpart for every tested $K$ and paired seed, although the gain decreases with increasing dimensionality. MMR gains vary non-monotonically and are nearly zero at $K=10$. This suggests greater sensitivity of deterministic extraction to the latent aspect representation.

\begin{table}[t]
\centering
\caption{Sensitivity to $K$: paired
online-over-static alignment gain
(mean $\pm$ 95\% CI).}
\vspace{-10pt}
\label{tab:k_sensitivity}
\small
\setlength{\tabcolsep}{4pt}
\begin{tabular}{@{}ccc@{}}
\toprule
$K$ & \textbf{Gumbel $\Delta$} & \textbf{MMR $\Delta$} \\
\midrule
5  & $0.4173\pm0.0247$ & $0.1915\pm0.0022$ \\
10 & $0.2971\pm0.0131$ & $0.0013\pm0.0020$ \\
15 & $0.1757\pm0.0081$ & $0.1919\pm0.0031$ \\
20 & $0.1280\pm0.0018$ & $0.1232\pm0.0019$ \\
25 & $0.0725\pm0.0029$ & $0.0720\pm0.0017$ \\
\bottomrule
\end{tabular}
\end{table}

\begin{table}[t]
\centering
\caption{Sensitivity to aspect-assignment
sharpness: paired online-over-static
alignment gain (mean $\pm$ 95\% CI).}
\vspace{-10pt}
\label{tab:tau_sensitivity}
\small
\setlength{\tabcolsep}{4pt}
\begin{tabular}{@{}lcc@{}}
\toprule
\textbf{Assignment} &
\textbf{Gumbel $\Delta$} &
\textbf{MMR $\Delta$} \\
\midrule
$0.5\tau^\star$ & $0.2400\pm0.0095$ & $-0.0004\pm0.0009$ \\
$\tau^\star$ & $0.2971\pm0.0131$ & $0.0013\pm0.0020$ \\
$2\tau^\star$ & $0.3150\pm0.0104$ & $-0.0035\pm0.0010$ \\
Hard assignment & $0.3265\pm0.0069$ & $0.0175\pm0.0044$ \\
\bottomrule
\end{tabular}
\end{table}

\subsubsection{Sensitivity to Assignment Temperature}
\label{sec:tau_sensitivity}

Fixing $K=10$, we vary the soft-assignment temperature around its calibrated value $\tau^\star$ and additionally consider hard aspect assignment. All other settings remain unchanged. Table~\ref{tab:tau_sensitivity} shows that \prefer-Gumbel retains positive online gains across all tested assignment settings. In contrast, MMR gains remain near zero under soft assignment and improve under hard assignment. Thus, Gumbel extraction is less sensitive to assignment sharpness in this experiment.

\end{document}